\documentclass[runningheads]{llncs}

\usepackage{eccv}

\usepackage{eccvabbrv}

\usepackage{graphicx}
\usepackage{booktabs}
\usepackage{algorithm}
\usepackage{algpseudocode}
\usepackage{enumerate}
\usepackage{multirow}
\usepackage{arydshln}
\usepackage[accsupp]{axessibility}  

\usepackage{hyperref}

\usepackage{orcidlink}

\begin{document}

\title{Learning Quantized Adaptive Conditions for Diffusion Models} 


\author{Yuchen Liang\inst{1}\orcidlink{0009-0003-7359-1067} \and
Yuchuan Tian\inst{2} \and
Lei Yu\inst{3} \and
Huaao Tang\inst{3}\and
Jie Hu\inst{3}\and
Xiangzhong Fang\inst{1}\and
Hanting Chen\inst{3}
}

\authorrunning{Liang et al.}

\institute{School of Mathematical Sciences, Peking University \and State Key Lab of General AI, School of Intelligence Science and Technology, Peking University \and
Huawei Noah’s Ark Lab\\
\email{ycliang@pku.edu.cn, chenhanting@huawei.com}}

\maketitle

\begin{abstract}
  The curvature of ODE trajectories in diffusion models hinders their ability 
  to generate high-quality images in a few number of function evaluations (NFE). 
  In this paper, we propose a novel and effective approach to reduce trajectory 
  curvature by utilizing adaptive conditions. By employing a extremely light-weight quantized 
  encoder, our method incurs only an additional 1\% of training parameters, eliminates 
  the need for extra regularization terms, yet achieves significantly better sample quality.
  Our approach accelerates ODE sampling while preserving the downstream task image editing capabilities of SDE techniques. Extensive experiments verify 
  that our method can generate high quality results under extremely limited sampling costs. 
  With only 6 NFE, we achieve 5.14 FID on CIFAR-10, 6.91 FID on FFHQ 64×64 and 3.10 FID on AFHQv2.
\end{abstract}

\section{Introduction}
\label{sec:intro}

Generative models based on ordinary differential 
equations (ODEs) have led to unprecedented success in various domains, 
including image synthesis \cite{ho2022classifier}, audio synthesis \cite{kong2020diffwave}, 3D reconstruction \cite{poole2022dreamfusion}, 
and video generation \cite{ho2022video}. These models transform a tractable noise distribution to the data distribution with differentiable trajectories to 
Early attempts faced limitations due to the requirement of simulating ODEs, 
which hindered their practical applicability. Recent advancements in score-based diffusion 
models \cite{song2019generative,song2020score,ho2020denoising,chung2023diffusion} avoid explicit ODE generation 
by employing a forward stochastic differential equation (SDE) process with 
an accompanying Probability Flow (PF) ODE. By numerically simulating the 
PF-ODE, diffusion models enable the generation of high-quality samples. This process 
requires multiple evaluations of a neural network leading to slow sampling
speed.

Efforts have been made to accelerate the sampling process and fall
into two main streams. One stream aims to build a one-to-one
mapping between the data distribution and the pre-specified
noise distribution \cite{berthelot2023tract,liu2022flow,luhman2021knowledge,salimans2022progressive,song2023consistency}, based on the idea of
knowledge distillation. However, these
methods require huge efforts on training. Training such a student model should carefully design the training details and takes a large amount
of time to train the model (usually several GPU days). Moreover, as distillation-based models directly
build the mapping like typical generative models, they suffer from the inability of interpolating between 
two disconnected modes \cite{salmona2022pushforward}. Therefore, distillation-based methods
may fail in some downstream tasks requiring such an interpolation. Besides, distillation-based models cannot
guarantee the increase of sample quality given more NFE and they have difficulty in likelihood evaluation. 
The other stream of methods focuses on designing faster numerical
solvers to increase step size while maintaining the sampling quality \cite{dockhorn2022genie,karras2022elucidating, liu2022pseudo,lu2022dpm,song2020denoising,zhang2022fast,zhao2024unipc}.  
Although these methods have successfully
reduced the number of function evaluations (NFE) from
1000 to less than 20, almost without affecting the sample
quality.  Nevertheless, these methods still face the challenge of the intrinsic truncation errors rooted in the curvature of the trajectory.
The sampling quality deteriorates sharply when the sampling budget is further limited.
\begin{figure}[tb]
  \includegraphics[width=\columnwidth]{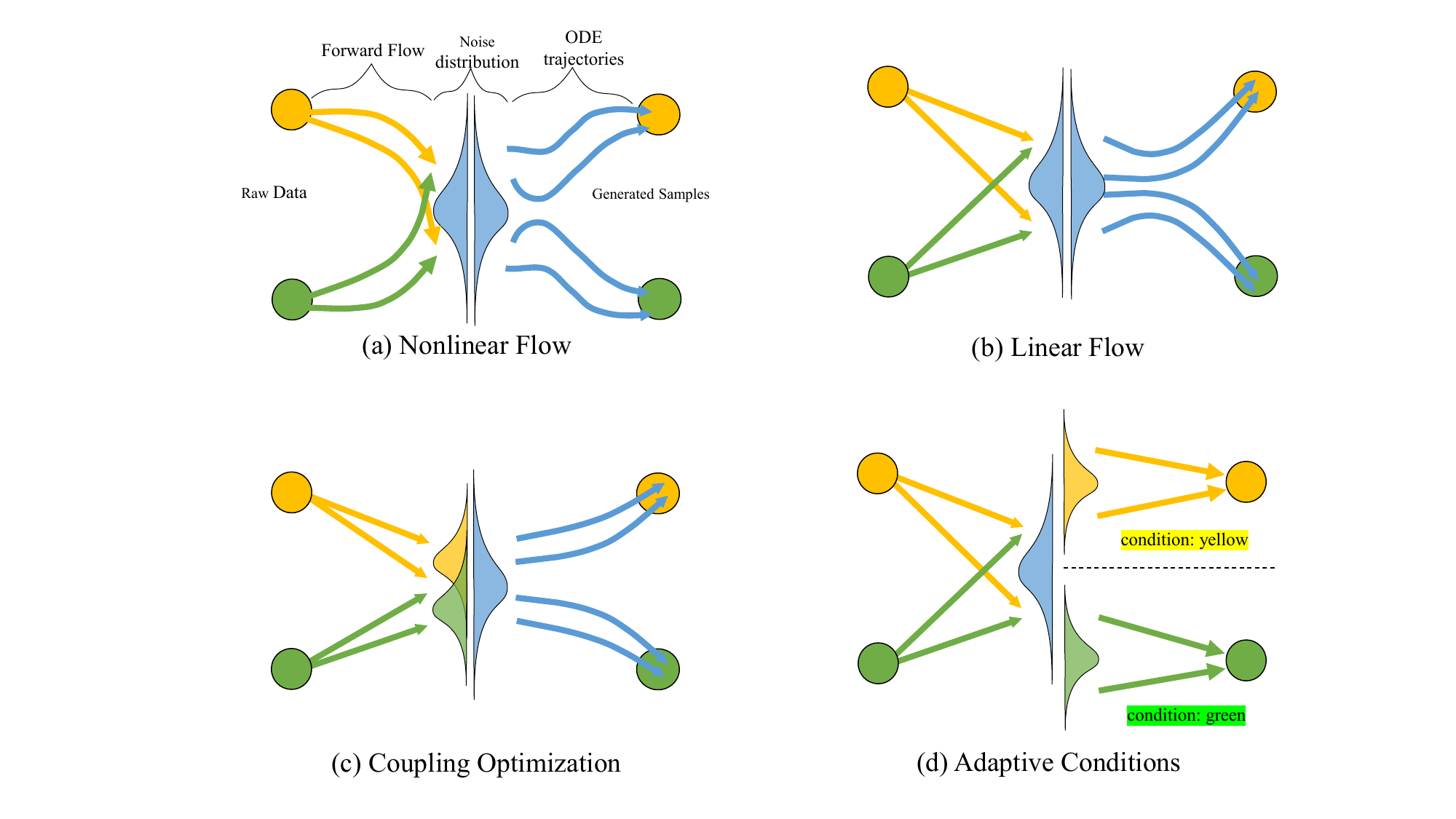}
  \caption{(a) Denoising diffusion models with nonlinear forward flow \cite{ho2020denoising,song2020score} have complex ODE trajectors. 
  (b) Linear Flow models \cite{liu2022flow,karras2022elucidating,song2020denoising,nichol2021improved} still have highly curved ODE trajectories.
  (c) Coupling optimization\cite{lee2023minimizing} methods try to reduce curvature by trajectory relocation, but are limited by the difficulty of keeping the noise distribution unchanged.
  (d) Adaptive Conditions untie the crossover between forward trajectories without compromising the full simulation accuracy.}
  \label{main_fig}
\end{figure} 

In order to reduce the trajectory curvature, Rectified Flow\cite{liu2022flow} provides a different understanding 
of diffusion models from the transport mapping perspective. The PF-ODE can be considered as the result of a 
rectification process, which unties the predefined forward coupling flows. In the case of score-based methods, independent coupling is 
utilized, which can be seen as a special case in the theory of Rectified Flow. These PF-ODE trajectories 
will be curved to avoid crossing. From this 
viewpoint, there are two key points to reduce the curvature of PF-ODE 
trajectories: the first is to use linear forward flow \cite{song2020denoising,karras2022elucidating,lipman2022flow}, 
and the second is to reduce the intersection of the forward trajectory to maintain 
straightness. 
Previous research aimed to reduce the intersection by optimizing coupling\cite{pooladian2023multisample,lee2023minimizing} via trajectory relocation. However, 
the optimization process is challenging to maintain the marginal noise distribution. 
And a series of stochastic sampling\cite{xue2024sa,lu2022dpm} and part of distillation techniques\cite{luo2024diff, yin2023one} designed for the score-based model are no longer available, 
which hinders the acceleration of the simulation process and impedes the applicability of these methods. 

In this paper, we propose a novel and effective approach for reducing the intersection. At the same time, the key properties of score-based models is retained.
Our approach is motivated by a straightforward analogy: when there is a 
large number of pedestrians needing to cross a road from both sides, city 
authorities would consider installing traffic lights instead of moving the road. Similarly, we lead 
the backward process with adaptively learned quantized conditions, which allows the backward PF-ODE trajectories 
to pass through intersection areas without 
the need for significant curvature or trajectory relocation.
A schematic diagram is shown in \cref{main_fig}.
Our contribution can be summarized as follows:
\begin{itemize}
  \item We investigate the relationship between the degree of forward flow intersection and the quality of the few-step sampling. 
  We provide theoretical support for the positive correlation of the intersection and the quality of the few-step sampling.
  \item We present a plug-and-play approach with a quite small additional training cost to reduce
  the degree of intersection, which is the first method that does not require trajectory relocation and additional regularization. 
  \item We conduct extensive comparison and ablation experiments on the CIFAR-10, MNIST, FFHQ and AFHQv2 to verify that 
  our method can achieve superior performance compared to the 
  original diffusion models in both few-step sampling and full sampling generation.
\end{itemize}
\section{Background}
\subsection{Diffusion Model}
Diffusion models set a stochastic differential equation (SDE) \cite{song2020score}
\begin{equation}
dx_t = \mu(x_t, t) dt + \sigma(t) dw_t \label{SDE}
\end{equation}
where $t\in [0,T], T>0$ is a fixed positive constant, $w_t$ is the standard Wiener process and $\mu(\cdot, \cdot)$ and $\sigma(\cdot)$ 
are the drift and diffusion coefficients respectively. 
We denote the distribution of $x_t$ as $p_t(x)$ and set $p_0(x)$ as the data distribution. A remarkable property of this SDE is
Remarkably, there exists a probability
flow ODE 
\begin{equation}
  dx_t = [\mu(x_t, t)- \frac{1}{2}\sigma(t)^2 \nabla_x \log p_t(x)] dt
  \end{equation}
  sharing the same marginals with the reverse SDE\cite{Maoutsa_2020}. 
  To simulate the PF-ODE, a U-Net $s_\theta(x, t)$ is usually trained to estimate the intractable score function $\nabla_x \log p_t(x)$ via score
  matching\cite{hyvarinen2005estimation,vincent2011connection}. 
The SDE in \cref{SDE} is designed such that $p_T$
is close to a tractable Gaussian distribution $\pi(x)$. 
During sampling stage, we will sample a noise image from $\pi(x)$ to initialize the empirical PF ODE
and solve it backwards in time with any numerical ODE solver.
\subsection{Rectified Flow}
From the view of transport mapping, Rectified Flow\cite{liu2022flow} offers
an alternative perspective which is fully explained under the
ODE scheme.
Let $\boldsymbol{X}=\{X_t:t\in [0,T]\}$ 
be any time-differentiable forward 
flow process that couples the data $X_0\sim p_{data}$ and the noise $X_T\sim p_{noise}$. Let $\dot{X_t}$ be the time derivative of $X_t$. The rectified flow induced from X is defined as 
\begin{equation}
  dZ_t = v^X(Z_t, t), \text{ with }Z_T=X_T
\end{equation}
where $v^X(z,t)=\mathbb{E} [\dot{X_t} | X_{t}=z ]$. And $v^X$ can be estimated by minimizing the conditional 
flow matching(CFM) objective
\begin{equation}
\label{1}
L_{\text{CFM}}(\theta) := \int_{0}^{T}w_t\mathbb{E}\| \dot{X_t} - v_{\theta}(X_t,t) \|^2dt, 
\end{equation}
where $w_t: (0,T)\rightarrow (0,+\infty)$ is a positive weighting sequence. 
The marginal preserving property\cite{liu2022flow}, 
$Law(Z_t)=Law(X_t)\quad \forall{t} \in [0,T],$
ensures that we can generate samples by simulating the reversed ODE fromtractable random variable $X_T$. With independent coupling $(X_0, X_T)$, 
different forward process correspond to different diffusion models:\\\\
\textbf{Nonlinear Flows}    
\begin{align}
        \text{VP\cite{ho2020denoising}} :& X_t=\alpha(t) X_0 +\sqrt{1-\alpha(t)^2} X_1, X_1\sim \mathcal{N} (0, I)\\
        \text{VE\cite{song2020score}} :& X_t=X_0 +\sqrt{t} \xi,  X_T \thickapprox \sqrt{T}\xi \sim \mathcal{N} (0, TI)
\end{align}
\textbf{Linear Flows} 
\begin{align}
    \text{EDM/DDIM\cite{karras2022elucidating,song2020denoising}} :& X_t=X_0 +t \xi,  X_T \thickapprox T\xi \sim \mathcal{N} (0, T^2I), \\
    \text{RectifiedFlow\cite{liu2022flow}} :& X_t=(1-t)X_0 +t X_1, X_1\sim \mathcal{N} (0, I).
\end{align}
Compared with nonlinear flow methods, linear flows have showed 
 a significant effect on sampling acceleration. 
\subsection{Coupling Optimization}
In order to further reduce the cost of ODE trajectory simulation, 
 Liu etc.\cite{liu2022flow} introduce an multistage optimization 
approach where the original coupling $(X_0, X_1)$ is substituted with rectified 
coupling $(Z_0, Z_1)$. Lee etc. \cite{lee2023minimizing} and Pooladian etc. \cite{pooladian2023multisample} adopt joint 
training to avoid additional training iterations and to prevent errors caused by 
ODE simulation. The joint training relies 
on the following bias-variance decomposition of the CFM loss:
\begin{equation}
  L_{\text{CFM}}=L_{\text{FM}} + V((X_0,X_1))
\end{equation}
where 
\begin{equation}
L_{\text{FM}} := \int_{0}^{T}w_t\mathbb{E}\|\mathbb{E}[\dot{X_t}|X_t]-v_{\theta}(X_t,t)\|^2 dt
\end{equation}
judges the accuracy of direction fitting and
\begin{equation}
V((X_0,X_1)) := \int_{0}^{T}w_t\mathbb{E}\|\dot{X_t}-\mathbb{E}[\dot{X_t}|X_t]\|^2 dt
\end{equation}
measures the intersection of forward flow. 

When $V((X_0,X_1))$ approaches zero, 
the curvature of rectified trajectories also tends to zero. Therefore, the coupling optimization 
can be performed jointly with direction fitting by minimizing $L_{\text{CFM}}$ 
without the need for simulation of ODE trajectories. 

However, it is challenging to maintain the marginal distribution for $X_0$ and $X_1$. 
Multisample flow matching \cite{pooladian2023multisample} constructs a doubly-stochastic matrix for the coupling ditribustion
, which limits by the batch size. And in the work\cite{lee2023minimizing} of Lee etc., they employ a 
reparamized noise encoder, which compromises to the error between the encoded 
distribution and the prior distribution. To address this challenge, we propose an alternative approach 
to reduce the intersection of the forward flow. 

\section{Adaptive Conditions}
We discriminate the forward trajectories with different adaptive conditions represented by $Y$, which can be considered as 
pseudo-labeling of the image data $X_0$ and independent of the noise $X_1$. 
The allocation of conditions is carried out by an autoencoder $q_{\phi}(x_0, y) = p_{data}(x_0)q_{\phi}(y|x_0)$ 
and leaded by the conditional CFM loss
\begin{equation}
    \label{2}
    L_{\text{CFM}}=\int_{0}^{T}w_t\underset{X_0, Y\sim q_{\phi}}{\mathbb{E}}\| \dot{X_t} - v_{\theta}(X_t,t, Y) \|^2dt,
\end{equation}
which is an expection of unconditional $L_{\text{CFM}}$ with different conditions $Y$. 
\subsection{Discretization Error Control}
We present theoretical evidence demonstrating that for linear flows $L_{CFM}$ can effectively control the influence of discretization error accumulation
on the generation quality and address the inconsistencies that impact distillation efficiency. 
\begin{theorem}
\label{th:3.1}
Let $S_{t, y}\sim \tilde{p}_t $ be a simulation of $X_t|_{Y=y} \sim p_{t,y} $ and 
\begin{equation}
  S_{t,y} - \Delta t~v_{\theta}(S_{t,y}, t, y) =: d_{t,\Delta 
  t}(S_{t,y}, y)\sim \tilde{p}_{t-\Delta t,y},
\end{equation}
be the one-step further simulation of $X_{t-\Delta t, y}\sim p_{t-\Delta t,y}$. And $d_{t,\Delta t}(S_{t,Y},Y)\sim \tilde{p}_{t-\Delta t}$ denote the overall simulation of $X_{t-\Delta t}\sim p_{t-\Delta t}$. Then we can
contral the Wasserstein distance 
\begin{align}
W(\tilde{p}_{t-\Delta t}, p_{t-\Delta t}) &\leq [\mathbb{E}_Y W^2(\tilde{p}_{t-\Delta t,Y}, p_{t-\Delta t,Y})]^{\frac{1}{2}}\\
&\leq \Delta t\cdot l_{\text{CFM}}
(t)+ L[\mathbb{E}_Y
W^2(\tilde{p}_{t,Y}, p_{t,Y})]^{\frac{1}{2}}& \label{3}
\end{align}
where $L$ is the Lipschitz constant for $d_{t,\Delta t}(\cdot, y)$ and 
\begin{equation}
  l_{\text{CFM}}(t)=[\mathbb{E}\| \dot{X_t} - v_{\theta}(X_t,t, Y) \|^2]^{\frac{1}{2}}
\end{equation}
is a component of the optimization objective defined in \cref{2}
\begin{equation}
  L_{\text{CFM}}=\int_{0}^{T}w_t [l_{\text{CFM}}(t)]^2dt.
\end{equation}
\end{theorem} 
\begin{proof}
  The proof parallels the proof of the Wasserstein distance upper bound 
  for score-based generative models \cite{kwon2022score}. A tighter upper bound can also be obtained 
  following the technique provided by \cite{kwon2022score}.  
  We provide a complete proof in supplementary materials.
  \end{proof}
And for unconditional $L_{\text{CFM}}$ there is a more intuitive and concise version of \cref{3}. 
\begin{equation}
W(\tilde{p}_{t-\Delta t}, p_{t-\Delta t}) \leq
\underbrace{\Delta t}_{\text{step size}}\cdot \underbrace{l_{\text{CFM}}
(t)}_{\text{new error}}+ \underbrace{L}_{\text{amplifying coefficient}}\cdot\underbrace{W(\tilde{p}_{t}, p_{t})}_{\text{original error}}
\end{equation}

\cref{th:3.1} have shown that a smaller loss $L_{\text{CFM}}$ can provide quality assurance 
for single-step update with Euler solver, and more advanced deterministic and stochastic solvers can be regarded as 
corrections based on this result. Specifically when setting $\Delta t=t$ the theorem shows that the gap between the predicted image 
distribution and the groudtruth distribution can be bounded by $L_{\text{CFM}}$. The experiments of \cite{lee2023minimizing} and 
\cite{shao2023catchup} showed that a smaller $L_{\text{CFM}}$ can improve the efficiency and final performance of 
distillation. \cite{pooladian2023multisample} also clarified that $L_{\text{CFM}}$ can reduce the variance of 
stochastic gradients, which provides a more stable training process and convergence speed.  For the score-base diffusion model, $L_{CFM}$ cannot be optimized to zero 
because of the intersection of forward flows. 
Fortunately, as shown in \cref{t1}, coupling optimization and adaptive condition provide more 
optimization space and when the $L_{CFM}$ approaches zero, the ode trajectory tends to be completely straight. 
\begin{table}[htb]
  \caption{The result of the optimized $L_{\text{CFM}}$ and average curvature of ODE trajectories after training by 
  15M images drawn from the dataset Cifar10. We use the same model configs and sampler as Reparamized Noise Encoder(RNE)\cite{lee2023minimizing} 
  with a prior regularization $\beta=20$ for a fair comparison.}
  \label{t1}
  \centering
  \setlength{\tabcolsep}{3mm}{
  \begin{tabular}{@{}lcccc@{}}
    \toprule
    \multirow{2}{*}{Method} & \multirow{2}{*}{$L_{\text{CFM}}\downarrow$}  & \multirow{2}{*}{Curvature$\downarrow$} & \multicolumn{2}{c}{FID$\downarrow$}\\
    \cmidrule(r){4-5}
    & & & Heun w/5 NFE & RK45 \\
    \midrule
    RectifedFlow  & 0.176 & 0.46 & 37.19 & 2.66 \\
    RNE & 0.153 & \textbf{0.38} & 24.54 & \textbf{2.45}\\
    Adaptive Condition &  \textbf{0.132} & 0.40 & \textbf{19.68} & \textbf{2.43} \\
  \bottomrule
  \end{tabular}}
\end{table}
\subsection{Quantized Condition Encoder}
For learning the adaptive conditions, we have the flexibility to use any variational autoencoder (VAE)\cite{kingma2013auto}. However, unlike reparamized noise encoder\cite{lee2023minimizing}, the condition coding space is decoupled from the noise space. This decoupling provides more freedom in choosing the encoding strategy. We chose to use a quantized encoder because it doesn't suffer from posterior collapse, which is a common issue in other VAEs. A quantized encoder with a sufficiently large coding space can handle high-resolution image reconstruction effectively.

Typically, a visual tokenization generative model\cite{van2017neural, esser2021taming} requires an additional sequence model to learn the distribution of the quantized feature map. To avoid introducing extra training and inference costs, we use a single quantized vector instead of a quantized feature map. This significantly reduces the coding space dimensionality by several orders of magnitude. However, it also allows us to use a lightweight encoder( smaller than 0.8M), and the empirical distribution of the code vectors can be easily collected online.

We first encode the image $x$ into a $d$-dimensional vector $y=E_{\phi}(x)$ 
via a neural network encoder, and then discretize it using 
a finite scalar quantization\cite{mentzer2023finite} with some minor modifications.
For each channel $y_i$, we use a function $f$ to restrict 
its output to $L$ discrete integers. We choose $f$ to be 
$f: y\mapsto \left\lfloor L \cdot \sigma(y)\right\rfloor $ where $\sigma(y)=1/(1+e^{-y})$ is the sigmoid function, 
which is more symmetric compared to the original version $round(\left\lfloor L/2\right\rfloor tanh(y))$ when $L$ 
is even. As a result, the value space of $y_q=f(y)$ forms 
an implied codebook $C=\{0,1,\dots,L-1\}^d$, which is given by the product of these 
per-channel codebook sets, with $|C| = L^d$. Finally, we use the 
Straight-Through Estimator\cite{bengio2013estimating} to copy the gradients from the decoder 
input to the encoder output, which allows us to obtain gradients 
for the encoder. This can be implemented easily using the 
"stop gradient" (sg) operation as follows: 
$f_{STE}:y\mapsto y + sg(f(y) - y)$.

Each condition code $y_q\in C$ corresponds to a data slice, which can be viewed 
as a form of pseudo-labeling. We train diffusion models with quantized condition 
encoder just as in a conditional manner. We use two MLP layers to map the 
condition code to the same dimension as the time embedding and then add them 
together. Then, we input the combined embedding and noised image into the decoder. 
A visual schematic of our approach is shown in Figure \ref{frame}.
\begin{figure}[tb]
  \centering
  \includegraphics[width=\columnwidth]{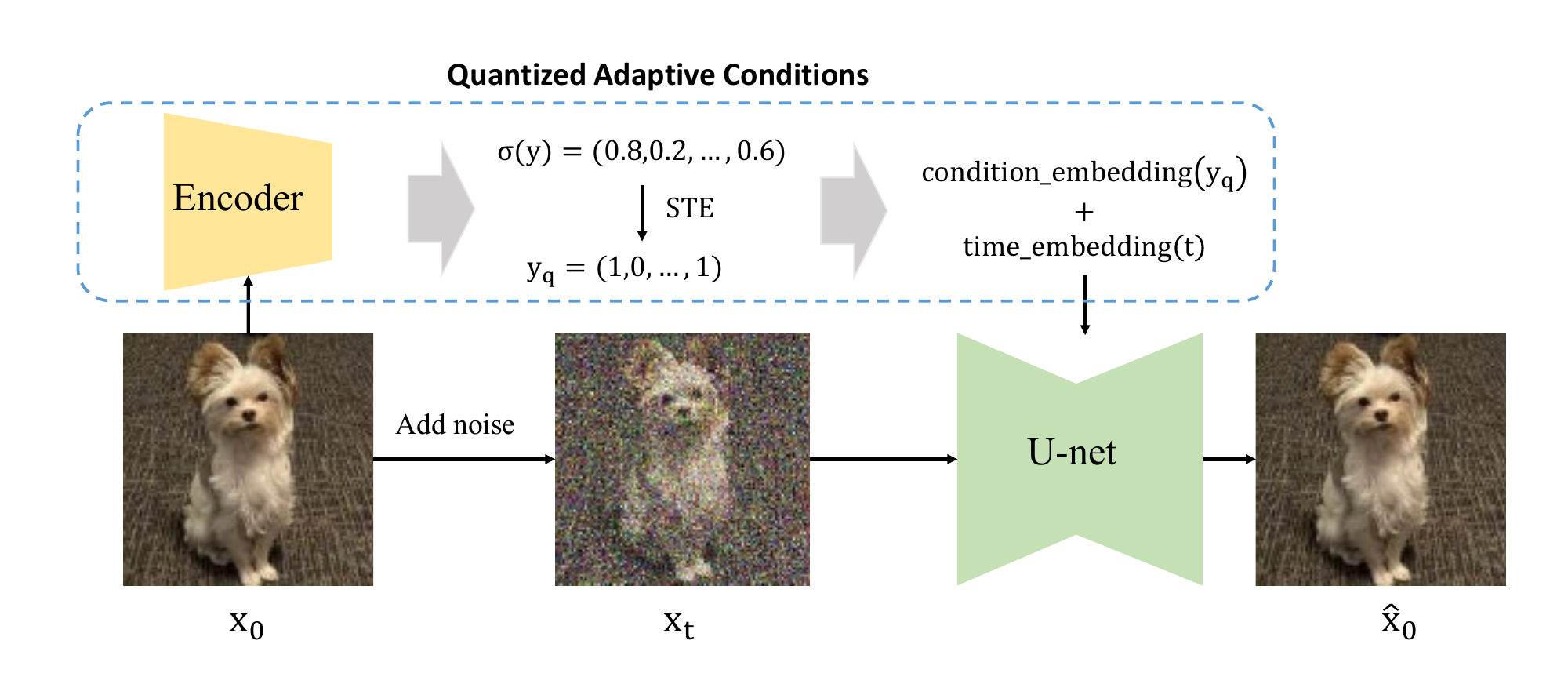}
  \caption{A visual schematic of our approach.}
  \label{frame}
\end{figure} 
\subsection{Online Sampling Weight Collection} Diffusion models, 
even when they converges, still have non-zero gradients. Therefore, 
exponential moving average (EMA) is often used to update the parameters more stably. 
However, this introduces a time inconsistency problem between the parameters 
and the condition sampling weight. To deal with this problem, we propose two  sampling weight collection strategies: 
\begin{itemize}
\item \textbf{Offline collection}: Also update the condition encoder with EMA and then collect the sampling weight for the training dataset with the final checkpoint of the EMA encoder. 
\item \textbf{Online collection}: Inspired EMA normalization\cite{cai2021exponential}, we use 
EMA synchronized with the model parameters to update the sampling 
weight for each mini-batch during the training process. 
\end{itemize}
As shown in \cref{t2}, the online sampling weight exhibits significantly better performance on both few step generation with Heun's method and full simulation generation with an adaptive step solver RK45.
\begin{table}[htb]
  \caption{A comparison of online and offline collection, based on ODE trajectory of RectifiedFlow with a condition codebook size $2^{20}$}\label{t2}
  \centering
  \setlength{\tabcolsep}{3mm}{
  \begin{tabular}{@{}lccc@{}}
    \toprule
    \multirow{2}{*}{Collection Strategy}& \multicolumn{3}{c}{FID$\downarrow$}\\
    \cmidrule(r){2-4}
    & Heun w/ 5 NFE & Heun w/ 7 NFE & RK45\\
    \midrule
    Offline  & 21.07 & 14.40 & 5.96 \\
    Online   & \textbf{19.68} & \textbf{13.22}  & \textbf{2.61}\\
  \bottomrule
  \end{tabular}}
\end{table}

\subsection{Training and Sampling}
We train the condition encoder and image denoiser networks jointly. Following previous works, we 
update the  model parameters and the online 
sampling weights with EMA. The complete training process of quantized adaptive conditions can be summarized as \cref{alg1}.

\begin{algorithm}[ht]
  \caption{Quantized Adaptive Conditions Training}\label{alg1}
  \begin{algorithmic}[1]
    \Require 
    dataset $\mathcal{D}$, noise level distribution $p_{\sigma}$, 
    encoder initial parameter $\phi$, denoiser initial parameter $\theta$, 
    loss weighting $\lambda(\cdot)$, learning rate $\eta$, EMA decay rate $\mu$   
    \State $\theta^{-}=\theta$                     \Comment{Copy initial parameter to EMA model}
    \State $w \leftarrow 0_{1\times |\mathcal{C}|}$ \Comment{initiate sampling weight with zeros}
    \Repeat 
    \State Sample $x\sim \mathcal{D}$, $t\sim p_{\sigma}$ and $z\sim \mathcal{N}(0,I)$
    \State$y\leftarrow E_\phi(x)$
    \State$y_q \leftarrow  y + sg(\left\lfloor L \cdot \sigma(y)\right\rfloor-y)$
    \Comment{Compute condition codes by STE}
    \State$x_t\leftarrow x+tz$
    \State$\mathcal{L}(\theta, \phi)\leftarrow \lambda(t)\|x-D_{\theta}(x_t,t,y_q)\|^2$
    \State$\theta \leftarrow \theta -\eta \frac{\partial \mathcal{L}}{\partial \theta}$,
    $\phi \leftarrow \theta -\eta \frac{\partial \mathcal{L}}{\partial \phi}$ 
    \State$\theta^{-} \leftarrow \mu \theta^{-} + (1-\mu) \theta$
    \State $w_{batch} \leftarrow \text{Count}(\text{Index}(y_q)) $      \Comment{Collect batch sampling weights of code indices}
    \State$w \leftarrow \mu w + (1-\mu) w_{batch} $ \Comment{Update sampling weights by EMA}
    \Until{convergence}\\
    \Return{$\theta^{-}, w$}
  \end{algorithmic}
\end{algorithm}

During the sampling stage, we also follow a procedure that is similar to regular conditional diffusion models. Here is a breakdown of the steps:
\begin{enumerate}
  \item Code Index Selection: Randomly select an condition index based on the collected sampling weights. This index corresponds to a specific condition code.
  \item Noise Sampling: Sample a random noise from the distribution of $X_T$ or an approximation of it. This sampled noise serves as the initial point for the reverse ODE or SDE process.
  \item Score Function Estimation: Utilize the denoiser's output to compute an estimation of the score function conditioned on the selected condition code. This estimation helps guide the sampling process.
  \item Reverse Process Simulation: Use a deterministic or stochastic solver to numerically simulate the reverse process. This solver propagates the initial point backward in time, following the dynamics defined by the reverse ODE or SDE. The result is a generated sample that incorporates the selected condition code and the sampled noise.
\end{enumerate}
By incorporating adaptive conditions, the reverse process simulation ensures that the generated samples follow the desired dynamics defined by the reverse ODE or SDE.
\begin{table}[tb]
  \centering
    \caption{FIDs with different settings of condition encoder.} 
    \label{t3}
        \begin{tabular}{@{}lcccc@{}}
          \toprule
          Dataset &  CIFAR-10 & \multicolumn{2}{c}{MNIST}\\
          Solver & Heun & \multicolumn{2}{c}{Euler} \\
          NFE & 9  & 4 & 16\\
          \midrule
          w/o condition     & 12.92 & 31.67 & 11.16 \\
          VAE(optimal weight)  & 12.23 & 10.75 & 5.63 \\
          FSQ($L=8,d=4$) & 11.07 & 8.50 & 1.81\\
          FSQ($L=4,d=6$) & 10.85& 7.49 & 1.71\\
          FSQ($L=2,d=12$) & \textbf{10.83} & \textbf{6.36} & \textbf{1.19} \\
          \bottomrule
          \end{tabular}
\end{table}
\begin{table}[htb]
  \centering
  \caption{FID on CIFAR10 with Heun's method when scaling the
  codebook size}\label{t4}
  \setlength{\tabcolsep}{4mm}{
      \begin{tabular}{@{}lccc@{}}
        \toprule
        Codebook size$\backslash$NFE & 5 & 7 & 9 \\
        \midrule
        w/o a codebook     & 37.19 & 18.60 & 12.92 \\
        $|\mathcal{C}|=2^{12}$ & 21.86 & 13.78 & 10.83 \\
        $|\mathcal{C}|=2^{14}$ & 21.38 & 13.74 & 10.71 \\
        $|\mathcal{C}|=2^{16}$ & 20.18& 13.22 & 10.60 \\
        $|\mathcal{C}|=2^{18}$ & 20.13 & 13.17 & 10.34 \\
        $|\mathcal{C}|=2^{20}$ & \textbf{19.68} & \textbf{13.11} & \textbf{10.28} \\
        \bottomrule
        \end{tabular}}
\end{table} 
\section{Experiments}
\subsection{Quantized Condition Encoder}\label{sec:4.1}
\Cref{t3} demonstrates the improvement of generated image quality 
achieved by different condition encoders. We use Rectified Flow\cite{liu2022flow} as the baseline and adopt the same training configuration and sampling procedure 
temporarily in this subsection for a fair comparison. Following RNE\cite{lee2023minimizing}, 
we choose the encoder network with $\frac{1}{4}$ channels and $\frac{1}{2}$ blocks of the counterpart U-net network. 
And it's worth mentioning that we only need the encoder part of the U-net network, so we used even fewer extra parameters than RNE\cite{lee2023minimizing}.
Regardless of the type of 
condition encoders used, quantized adaptive conditions can consistently enhance the few-step generation 
performance. The Finite Scalar Quantizations (FSQ) excels in avoiding prior collapse and outperforms vanilla VAE significantly.

For a fixed codebook size of $2^{12}$, smaller levels yield better performance. This is because the quantized code forms an information bottleneck for image reconstruction, and smaller levels provide more dimensions to capture essential details. Therefore, we set the level $L=2$ to obtain a higher-quality representation of the images during the generation process.

As shown in \cref{t4}, scaling the codebook size can always demonstrate 
better generation quality. Even when the codebook size reaches $2^{20}$, 
which exceeds the number of training samples by far, the sampling weights 
as a model buffer only occupy a parameter count of 1M. 
\begin{figure}[tb]
  \centering
  \includegraphics[width=\columnwidth]{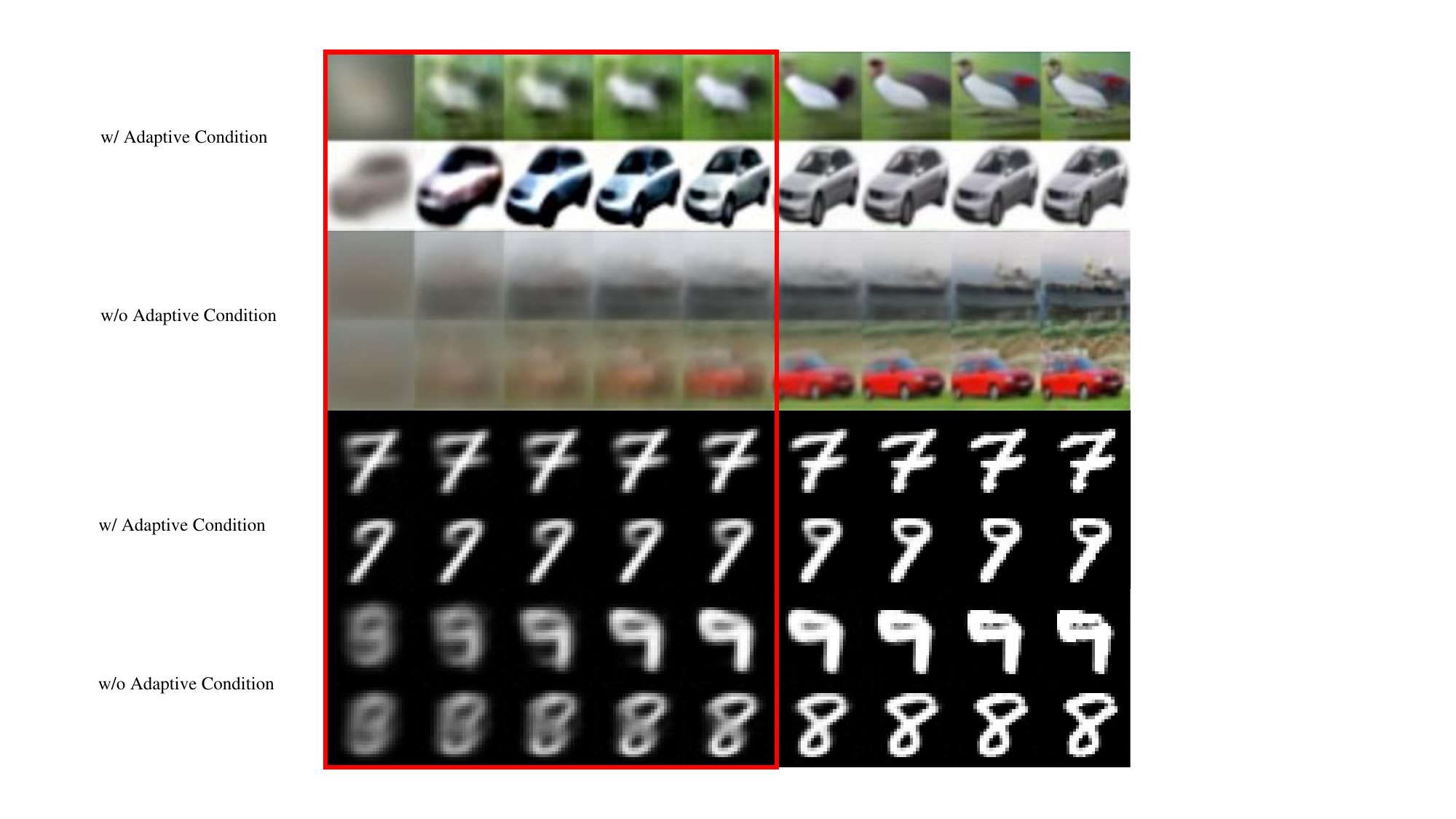}
  \caption{Visualization of intermediate samples. Adaptive conditions allow for sharper initial predictions at high noise level, as indicated
  by red boxes.}
  \label{traj}
  \end{figure} 
\subsection{Working with Reparamized Noise Encoder}\label{sec:4.2}
Since the code space of condition is independent of the noise space, the quantized condition encoder can be 
applied with the reparamized noise encoder together to further reduce the curvature.
These two approaches can cooperate efficiently with each other by sharing the encoder network backbone.
As shown in \cref{t5}, the best results are obtained by using both methods.
\begin{table}[!htb]
  \centering
  \caption{Applying reparamized noise encoder (RNE) and quantized adaptive conditions (QAC) at the same time.}
  \label{t5}   
  \setlength{\tabcolsep}{4mm}{
  \begin{tabular}{@{}lccc@{}}
        \toprule
        NFE & 5 & 7 & 9 \\
        \midrule
        RectifiedFlow    & 37.19 & 18.60 & 12.92 \\
        w/ RNE only & 24.54 & 13.99 & 10.21 \\
        w/ QAC only & 19.68 & 13.11 & 10.28 \\
        w/ RNE \& QAC & \textbf{18.20} & \textbf{12.11} & \textbf{9.58} \\
        \bottomrule
        \end{tabular}}
        \vspace{-1cm}
\end{table} 
\begin{figure}[tb]
  \centering
  \includegraphics[width=\columnwidth]{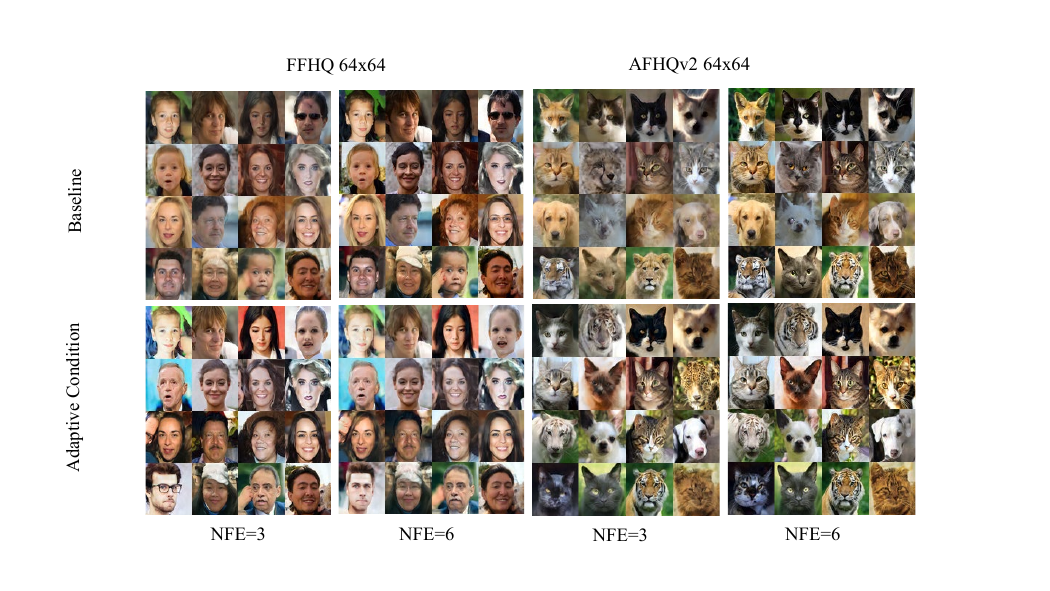}
  \vspace{-1cm}
  \caption{Qualitative comparison between our method and baseline on CIFAR-10, FFHQ and AFHQv2.}
  \label{samples}
  \end{figure} 
  \begin{table}[h!t]
    \caption{Performance comparisons on CIFAR-10.}
     \centering
     \setlength{\tabcolsep}{7mm}{
     \begin{tabular}{lcc}
       \toprule
       Model & NFE$\downarrow$ & FID$\downarrow$ \\
       \midrule
       \multicolumn{3}{l}{\textbf{Diffusion Models}}\\
       \midrule
       DDPM~\cite{ho2020denoising} & 1000 & 3.17 \\
       NCSN++~\cite{song2020denoising} & 2000 & 2.20  \\
       VDM~\cite{kingma2021variational} & 1000 & 7.41 \\
       LSGM~\cite{vahdat2021score} & 147 & 2.10 \\
       RectifiedFlow~\cite{liu2022flow} & 127 & 2.58\\
       EDM~\cite{karras2022elucidating} & 35 & 2.01 \\
       RNE~\cite{lee2023minimizing} & 9 & 8.66 \\
       \midrule
       \multicolumn{3}{l}{\textbf{Accelerated Sampler}}\\
       \midrule
       DDIM~\cite{song2020denoising} & 10 &  15.69 \\
       DPM-solver\cite{lu2022dpm} & 8 & 10.30 \\
       DPM-solver++\cite{lu2023dpmsolverfastsolverguided} & 6 & 11.85 \\
       UniPC~\cite{zhao2024unipc} & 6 & 11.10 \\
       DEIS~\cite{zhang2022fast} & 6 & 9.40 \\
       DPM-Solver-v3~\cite{zheng2023dpmsolverv3} & 6 & 8.56 \\
       AMED~\cite{zhou2023fast} & 6 & 6.63 \\
       \midrule
       QAC(Ours) & 20 & 2.10  \\
       QAC(Ours) & 6 & 5.14  \\
       \bottomrule
     \end{tabular}}
     \vspace{-0.5cm}
    \label{tab:cifar10_baseline}
 \end{table}
\begin{table}[!htb]
  \centering
  \caption{Sampling from Quantized Adaptive Conditions(QAC) with iPNDM(afs).}
  \label{t6}
  \setlength{\tabcolsep}{3mm}{
      \begin{tabular}{lcccc}
        \toprule
        NFE & 3 & 4 & 5 & 6 \\
        \midrule
        \multicolumn{5}{l}{\textbf{CIFAR10-32x32}}\\
        \midrule
        EDM\cite{karras2022elucidating} & 26.21 & 14.43 & 8.17 & 5.32  \\
        QAC & \textbf{21.06} & \textbf{12.48} & \textbf{7.50} & \textbf{5.14} \\
        \midrule
        \multicolumn{5}{l}{\textbf{FFHQ-64x64}}\\
        \midrule
        EDM\cite{karras2022elucidating} & 27.94 & 19.90 & 13.30 & 8.50  \\
        QAC & \textbf{22.92} & \textbf{15.34} & \textbf{9.42} & \textbf{6.91} \\
        \midrule
        \multicolumn{5}{l}{\textbf{AFHQv2-64x64}}\\
        \midrule
        EDM\cite{karras2022elucidating} & 15.60 & 8.58 & 5.55 & 3.78  \\
        QAC & \textbf{8.32} & \textbf{5.04} & \textbf{3.51} & \textbf{3.10} \\
        \bottomrule
        \end{tabular}}
\end{table} 
\subsection{Comparison with state-of-the-arts}\label{sec:4.3}
\Cref{t6} shows the unconditional synthesis results of our
approach on real-world image datasets including CIFAR-10\cite{krizhevsky2009learning} at $32\times32$
resolution, 
FFHQ\cite{karras2019style} and AFHQv2\cite{choi2020stargan} at $64\times64$ resolution. Considered as a 
one-session training approach, in \cref{tab:cifar10_baseline} we compare 
our method with other one-session training methods including a variety of diffusion models such as  NCSN++, DDPM, EDM, Rectified Flow and RNE and the accelerated sampling techniques. 
We train quantized adaptive conditions with a codebook size $2^{20}$ under the same training configuration as EDM\cite{karras2022elucidating} and adopt accelerated sampler 
iPNDM\cite{zhang2022fast} with the polynomial time schedule\cite{karras2022elucidating} and analytical first step \cite{dockhorn2022genie}. See the supplementary materials for details about the training configuration and sampling process. 
In the \cref{t6}, we can see that the
performance gap between our method and the baseline is
huge when the sampling budget is limited. For instance, our
method achieved an FID score of 8.32 on AFHQv2, which
is significantly better than the baseline’s score of 15.60 when
NFE is 3. Our method exhibits
superior sample qualities across all NFE, even in the case of full sampling. 
See \cref{samples} for visual
comparison. Additional qualitative results are provided
in supplementary materials.
\begin{figure}[tb]
  \centering
  \includegraphics[width=0.9\linewidth]{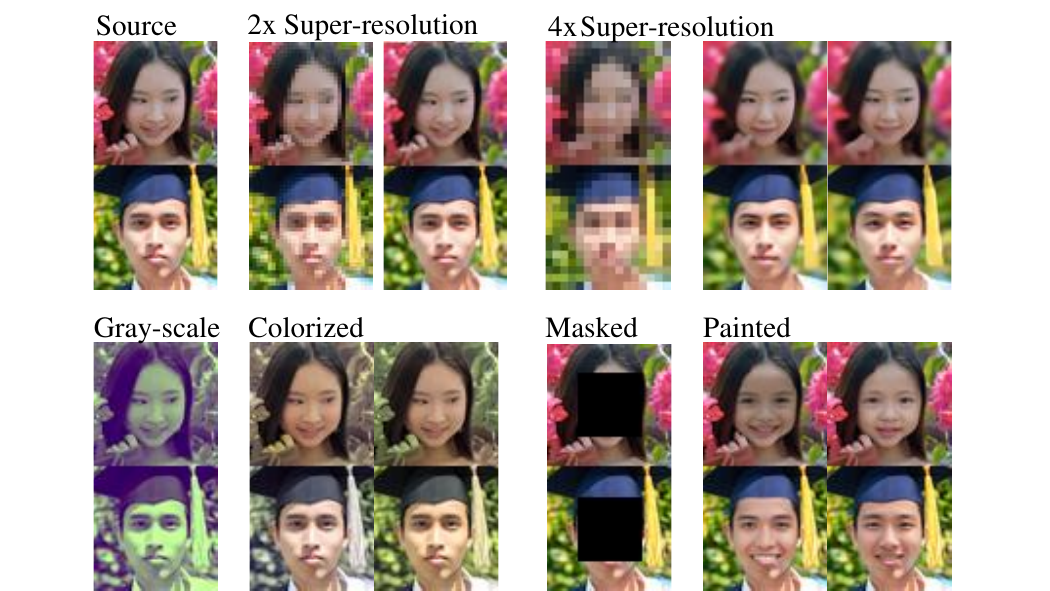}
   \caption{Our method allows SDE-based zero-shot image editing applications such as
   super-resolution, colorization and inpainting. In the experiments, we used Karras's schedule with steps $N=40$.}
   \label{fig:SDEdit}
\end{figure}
\subsection{Zero-Shot Image Editing}\label{sec:4.4}
Comparing with coupling operation, adaptive conditions does not affect the use of SDE-based technologies such as SDEidt. 
Note that $$p(x_{t_{n+1}}|x_{t_{n}})=p(x_{t_{n+1}}|x_{t_n}, y)p(y|x_{0})p(x_{0}|x_{t_n}).$$  According to the \cref{edit}, shown in \cref{fig:SDEdit}, our method is competent for various image editing tasks such as super-resolution, colorization and inpainting. 
\begin{algorithm}[tb]
  \caption{Zero-Shot Image Editing}\label{edit}
  \begin{algorithmic}[1]
    \Require 
    Denoising model $D$, condition code encoder $E$, time steps $\{t_n\}$, reference image $z$
    \State $z \leftarrow A^{-1}[(Az)\odot(1-\Omega)+0\odot\Omega]$       
    \State $x \leftarrow z$ 
    \For{$t=t_1$ to $t_N$}
    \State $y \leftarrow E(x)$
    \State Sample $x\sim\mathcal{N}(x, t^2\bf{I})$ 
    \State $x \leftarrow D(x, t, y)$
    \State $x \leftarrow A^{-1}[(Az)\odot(1-\Omega)+(Ax)\odot\Omega]$ 
    \EndFor\\
    \Return{x}
  \end{algorithmic}
\end{algorithm}
\section{Discussion and Limitations}
    One limitation is that for our encoding space
    , we use a code vector as adaptive condition for practical conveniences: sampling is easily implemented
    and the code sampling weight can conveniencely be collected. However, we can not fully 
    reconstruct the image in one step from such a short binary vector. 
    To further enhance the expressive power 
of the encoder, we believe that a feature map can be used 
instead of a code vector just like most visual tokenization works 
\cite{van2017neural, esser2021taming}.
And an additional network will be considered to generate condition codes 
instead of sampling weight collection for the case of a larger encoding space 
and conditional control generation. In additional, current sampling time 
schedule for diffusion models usually focus on low level of 
noise\cite{karras2022elucidating}. However, our method has the ability to 
reconstruct images with higher level of noise.  
Sampling method more suitable for our method is yet to be discovered. 
And further extension of our method to distillation or fine-tuning techniques 
is also highly anticipated.
\section{Conclusion}
In this paper, we show the degree of forward flow intersection will directly impact the generative performance of few-step sampling. 
We present a efficient plug-and-play method with a quite small additional 
training cost to reduce the average reconstruction loss, which is the first 
method that does not require trajectory relocation and additional regularization. Our approach preserves the critical properties 
of score-based models and is unique and complementary to other acceleration 
methods. We demonstrate that our approach improves the sample quality 
with a significantly reduced sampling budget. 
\par\vfill\par

%
%
\bibliographystyle{splncs04}
\bibliography{main}
\end{document}